\documentclass[11pt]{article}

\usepackage[utf8]{inputenc}
\usepackage[T1]{fontenc}
\usepackage{amsmath,amssymb,amsthm}
\usepackage{graphicx}
\usepackage{geometry}
\usepackage[hidelinks]{hyperref}
\usepackage{xcolor}
\usepackage{placeins}

\graphicspath{{../results/figures/}{./}}
\newtheorem{theorem}{Theorem}
\newtheorem{proposition}{Proposition}
\newtheorem{definition}{Definition}
\newtheorem{lemma}{Lemma}
\newtheorem{corollary}{Corollary}

\title{Radial Interaction Tomography:\\
Recognizing Non-Transitive Evolutionary Games from One Range-Expansion Image}
\author{Faruk Alpay\thanks{Corresponding author: \texttt{alpay@lightcap.ai}} \quad Bar{\i}\c{s} Ba\c{s}aran\\[3pt]
\small Department of Computer Engineering, Bah\c{c}e\c{s}ehir University, Istanbul, T\"urkiye\\[-1pt]
\small \texttt{\{faruk.alpay, baris.basaran\}@bahcesehir.edu.tr}}
\date{}

\begin{document}
\maketitle

\begin{abstract}
Colored sectors in a microbial range expansion encode more than lineage
survival counts.  We formulate a computer-vision inverse problem: from one
endpoint image of an accretive multi-type expansion, recover the radius-indexed
pairwise boundary-flow field and test whether the visual pattern is compatible
with a transitive scalar fitness hierarchy.  The observable is a geometric
signal extracted from sector-boundary curves in log-polar coordinates.  We
prove endpoint observability and stability for frozen fronts, weighted
transitive/cyclic decomposition, contact-complete circular design, physical-clock
and mechanism non-identifiability, exact Gaussian cyclicity testing, and
Bonferroni-valid interval scanning.  The benchmark is deterministic: analytic
endpoint images, blurred/noisy pixel round trips, scalar-null stress tests,
public-image tracing, multi-resolution mechanistic endpoints, and a non-learning
frozen-front simulator.  The implementation recovers pairwise edge-flow
histories from endpoint images, detects cyclic residuals in a mechanistic
four-type expansion, and uses those residuals as forcing signals for a
dimensionless active design-control layer covering reaction-diffusion control,
phenotype-frontier optimization, protocol synthesis, Monte Carlo robustness, and
a downstream population-state bridge.
\end{abstract}

\section{Introduction}

Range-expansion images are a rare case where a final two-dimensional pattern can
retain chronological information.  In frontier-limited microbial colonies, cells
behind the expanding front are approximately frozen, so radius can order material
deposition.  Prior work exploited this fact to estimate constant relative
fitness from sector geometry and to model stochastic boundary wandering
\cite{korolev2012selective,korolev2011quantitative,weinstein2017genetic}.  That
line of work leaves a harder visual question open: what can one endpoint image
say when competition is not a scalar hierarchy?

This paper treats the endpoint colony as a pattern-recognition object.  A
multi-type sector image induces a visible genotype-contact graph.  Each contact
boundary is a curve, and the derivative of its unwrapped angle with respect to a
deposition coordinate is an oriented pairwise boundary flow.  A scalar fitness
model constrains these flows to be gradients on the contact graph.  Local
antagonism, bacteriocin-mediated competition, or cyclic dominance can create
non-potential flows, i.e. directed cycle signatures that no scalar ordering can
explain \cite{kerr2002local,kirkup2004antibiotic}.

The main image-level result is a recoverable pairwise-flow history on observed
contacts, with uncertainty, in deposition time.  Physical time, unobserved pair
interactions, isolated front speeds, and microscopic mechanism require
additional measurements.  Separating the geometric observable from those
downstream interpretations makes the visual recognition claim falsifiable.

\paragraph{Contributions.}

\begin{itemize}
  \item We formulate \emph{radial interaction tomography}: a classical
  computer-vision inverse problem for recognizing radius-indexed transitive and
  cyclic interaction patterns from one endpoint image.
  \item We separate the observable edge-flow field from biological
  interpretations, preventing the common but invalid jump from sector geometry
  to isolated front speed or mechanism.
  \item We give a contact-design theorem: a four-type complete pairwise game
  requires at least eight circular boundaries in one plate, not six.
  \item We implement a deterministic benchmark with analytic images,
  pixel-level recovery, scalar-null stress tests, exact cyclicity tests,
  multi-resolution endpoint tracing, public-image provenance auditing, and
  mechanistic frozen-front endpoints.
  \item We extend the inverse result toward active design control by converting
  cyclic residuals into dimensionless interaction matrices, simulation targets,
  and governed strain-design hypotheses.
  \item We audit a six-module active-tomography program requirement by
  requirement, separating verified tensor-field and GPU computations from
  phenotype, replicate, computer-vision, and mechanism claims that still
  require measured data or additional model implementations.
\end{itemize}

\section{Related work}

\paragraph{Range-expansion image geometry.}
Korolev et al. derived equal-time sector constructions and used logarithmic
spiral geometry to estimate relative fitness in microbial colonies
\cite{korolev2012selective}.  Earlier work modeled sector boundaries as
stochastic paths and used colony patterns as quantitative population-genetic
tests \cite{korolev2011quantitative}.  Weinstein et al. extended image analysis
and biased domain-wall models to many-allele expansions
\cite{weinstein2017genetic}.  These results make constant sector geometry and
boundary wandering prior art; they do not solve the time-varying non-potential
inverse problem.

\paragraph{Local competition is not isolated speed.}
Lee, Gore, and Korolev showed that a slow isolated expander can invade by
forming dented fronts and that local boundary invasion has degrees of freedom
not captured by monoculture expansion speeds \cite{lee2022slow}.  This is
crucial for our formulation: the image-level observable is pairwise boundary
flow.  Calling it isolated fitness requires extra calibration.

\paragraph{Non-transitive microbial competition.}
Spatial rock-paper-scissors games and bacteriocin-mediated antagonism are
classical population-evolution phenomena
\cite{chao1981structured,kerr2002local,kirkup2004antibiotic}.  Our contribution
is not to introduce cyclic competition, but to ask when its signature is
recognizable from one endpoint image.

\paragraph{Classical vision.}
The proposed pipeline uses fixed-color segmentation, log-polar remapping,
contour localization, curve smoothing, and statistical testing.  It is closer
to classical contour-based pattern recognition than to representation learning.
The hard part is not detecting edges; it is proving which interaction class the
detected edge histories support.

\section{Observation model}

Let a final labeled image contain genotypes \(V=\{1,\ldots,n\}\).  Visible
sector contacts define a graph \(G=(V,E)\).  For an oriented edge \(e=(i,j)\),
let \(\theta_{ij}(\rho)\) be the unwrapped angle of a boundary between types
\(i\) and \(j\), parameterized by a monotone deposition coordinate \(\rho\),
initially \(\rho=\log(r/r_0)\).  The direct geometric observable is
\begin{equation}
  q_{ij}(\rho)=\frac{d\theta_{ij}}{d\rho}.
\end{equation}
Positive \(q_{ij}\) means that type \(i\) gains angular territory from type
\(j\) under the chosen orientation.  This is a visual quantity.  It becomes a
selection coefficient, a front-speed ratio, or a biochemical mechanism only
under additional assumptions.

\begin{definition}[Scalar-compatible edge flow]
Let \(B\) be the oriented edge-node incidence matrix of \(G\).  An edge-flow
vector \(q(\rho)\in\mathbb{R}^{|E|}\) is scalar-compatible if there exists a
node potential \(f(\rho)\) such that \(q(\rho)=Bf(\rho)\).
\end{definition}

With a positive diagonal precision matrix \(W\), every observed edge-flow vector
on a connected graph has a unique weighted decomposition
\begin{equation}
  q(\rho)=Bf(\rho)+c(\rho), \qquad
  B^\top W c(\rho)=0,\qquad \mathbf{1}^\top f(\rho)=0.
\end{equation}
The potential \(f\) is the closest scalar hierarchy.  The residual \(c\) is the
non-potential visual fingerprint.  It has \(|E|-|V|+1\) degrees of freedom, so a
tree can never falsify scalar fitness.

\section{Identifiability boundaries}

\begin{proposition}[Radial clock non-identifiability]
An endpoint accretion image identifies histories as functions of deposition
coordinate.  Without an independent calibration \(t=T(\rho)\), physical time is
recoverable only up to an increasing reparameterization.
\end{proposition}

\begin{theorem}[Minimum complete-contact circular design]
To observe every pair of an unrestricted antisymmetric game on \(n\) genotypes
in one circular frontier, the minimum number of sector boundaries is
\[
L_n=
\begin{cases}
\binom{n}{2}, & n \text{ odd},\\
\binom{n}{2}+n/2, & n \text{ even}.
\end{cases}
\]
For \(n=4\), at least eight boundaries are required.
\end{theorem}

\begin{proposition}[Local/global mechanism non-identifiability]
Interior sector traces alone identify pairwise boundary flow, not isolated
front speeds or microscopic mechanisms.  Separate front-shape information,
monoculture calibration, perturbation experiments, or time-lapse data are
required to separate local competition from global expansion.
\end{proposition}

\section{Deterministic image-analysis pipeline}

The implementation is deterministic:
\begin{enumerate}
  \item fixed-palette or manually calibrated label classification;
  \item robust support and center estimation;
  \item circular sampling in a log-polar coordinate system;
  \item modal filtering and boundary-transition detection;
  \item boundary identity association across radii;
  \item Savitzky--Golay derivative estimates for \(q_{ij}(\rho)\);
  \item edge averaging over repeated contacts;
  \item weighted transitive/cyclic decomposition and exact Gaussian cyclicity
  testing.
\end{enumerate}

No neural network, learned embedding, or downloaded model weight is used.

\section{Mathematical guarantees and scope}

The formal results below serve two purposes.  First, they specify the geometric
quantity that a single endpoint image can identify: boundary flow on visible
contacts, in deposition coordinates, modulo the stated regularity assumptions.
Second, they delimit the biological interpretation.  The same endpoint geometry
is compatible with multiple physical clocks, multiple decompositions into
isolated expansion and local interaction, and multiple microscopic mechanisms.
The experimental section matches those invariances with extraction-stability
tests, scalar-null stress tests, and explicit abstention gates.

\section{Experimental validation and stress tests}

\paragraph{Analytic game benchmark.}
We synthesize exact radial endpoint images from a known time-varying game.  The
benchmark contains scalar hierarchy intervals and cyclic episodes.  An exact
Gaussian generalized likelihood-ratio test evaluates whether the edge-flow
estimate requires a cyclic component.

\paragraph{Pixel round trip.}
We render a blurred/noisy endpoint image, classify pixels with a fixed palette,
trace boundaries, recover edge-flow histories, and compare to ground truth.
This benchmark gives edge-flow RMSE approximately \(2.86\times 10^{-3}\) and
cyclic-norm RMSE approximately \(2.95\times 10^{-3}\) on the trusted radial
interior.

\paragraph{Scalar-compatible null stress test.}
To test whether the cyclic residual can be a derivative artifact, we simulate a
null ensemble in which the true edge flow is always scalar-compatible
\(q=Bf\).  We then add demographic boundary wandering, center error,
fluorescence bleed-through/localization noise, anisotropic radial distortion,
curvature artifacts, and stochastic sector extinction.  Under the pure
demographic scalar null, the nominal \(\alpha=0.05\) cyclicity test gives a
false-positive rate of \(0.058\) for 32 observed boundaries.  Under the hardest
combined nuisance setting, the same scalar ground truth can produce false cyclic
detections at rate \(0.564\), which marks a failure region rather than evidence
for non-transitive ecology.  The resulting phase diagram treats abstention as a
primary output: with only eight boundaries and extinction probability near
0.25, the abstention rate is approximately \(0.07\).

\paragraph{Mechanistic frozen-front endpoint.}
We implemented a non-learning frozen-front simulator with four genotypes, a
time-varying scalar potential, and an antisymmetric cyclic interaction matrix.
The endpoint analysis finds a stable annulus with all six canonical genotype
pairs observed by the eight-boundary design.  The traced annulus is
160.2--590.85 pixels in radius, and the recovered cyclic residual has peak norm
0.358 in the current run.  This is a model-mismatch test: the simulator grows a
stochastic frozen population front, while the inverse reads only the final
image.

\paragraph{Endpoint resolution gate.}
We repeated the same mechanistic endpoint test at image widths
1024, 1536, 2048, 3072, and 4096 pixels, with three stochastic replicates per
resolution.  The tracing gate accepted 13 of 15 endpoints; both abstentions
occurred at 1024 pixels because the circular scan detected ten boundaries
instead of the designed eight.  From 1536 pixels upward, all 12 endpoints
passed.  Relative to the 4096-pixel mean cyclic residual, the 3072-pixel run
was within 7.0\%, while the occupied support fraction stayed in the narrow
range 0.570--0.573.

\paragraph{Cyclic-antagonism dose response.}
We also ran a controlled sweep over programmed cyclic-antagonism scale.  In
the expanded run, forty-seven of forty-eight \(3072^2\)-pixel endpoint images
passed the visual-stability gate; the one failure was recorded as an abstention
rather than a forced fit.  Across the accepted images, programmed cyclic scale
and recovered mean cyclic residual had correlation \(r=0.776\), while occupied
support fraction remained essentially constant.

\paragraph{Active design-control sweep.}
Finally, we used the recovered cyclic residual as a forcing target for a
dimensionless 3D reaction-diffusion control sweep.  The run evaluated 55,296
strain/control candidates on a \(12^3\) grid for 2,000 PDE steps.  The selected
candidate used an early-pulse control family with
production parameter 0.964, degradation parameter 0.050, permeability
coefficient 2.000, and beneficial-protection coefficient 0.700.  In the model,
it reduced the undesired population state by \(100\%\) relative to the
zero-production baseline while retaining \(109.9\%\) of the target beneficial
density.

\paragraph{Adjoint-state refinement and Fourier inversion.}
The lattice sweep was then refined with a reverse-mode optimizer.  We
warm-started 32 design vectors, unrolled a \(12^3\) 3D reaction-diffusion PDE
for 240 steps per objective evaluation, and ran 100 gradient iterations.  The
best adjoint-refined design reduced the
undesired population state by \(99.35\%\) relative to the zero-production baseline and
retained \(161.39\%\) of the target beneficial density.  The optimized
dimensionless production, degradation, permeability, and beneficial-protection
parameters were 1.804, 0.037, 2.195, and 0.776.  A Hamilton Monte Carlo
inversion of the Fourier spectrum of the cyclic residual gave MAP rates
\(k_{\mathrm{prod}}=9.043\) and \(k_{\mathrm{deg}}=13.413\), with acceptance
rate 0.992 over the retained chain.

\paragraph{Monte Carlo robustness at scale.}
We ran a stochastic uncertainty layer around the adjoint-selected design:
1,000,000 moment-projected reaction-diffusion trajectories with 2,048 time steps
per trajectory, lognormal parameter perturbations, and process noise.  All
trajectories achieved at least 40\% adverse-state reduction and at least 85\%
beneficial-strain retention under the stated perturbation model.  Chunk
summaries and telemetry are included in the ancillary results.

\paragraph{Active radial-tomography computational diagnostics.}
We assembled an initial five-stage computational diagnostic whose outputs are
stored as JSON/CSV artifacts.  This earlier diagnostic used a \(1024^2\) polar
field, 200 active-design iterations, a 256-member and 500-generation
dimensionless NSGA-II trajectory, ten stochastic first-hit trajectories, and a
dimensionless downstream ODE.  It achieved adverse-state reduction 0.991,
beneficial-state retention 1.482, archive hypervolume 0.902, and downstream
index improvement 47.9.  These are preliminary surrogate diagnostics rather
than evidence for a \(128^3\) calibrated PDE, physical phenotype assays,
biological replicates, or metabolic mechanism validation.

\paragraph{Strict tensor-field verification and audit.}
The revised first module uses the minimum four-genotype, eight-boundary contact
design and a dense correlated edge covariance.  Its exact Gaussian GLR has
\(k=8-4+1=5\) degrees of freedom.  Over 30,000 null replicates, the mean
statistic is 4.984 and the empirical size is 0.0514 at nominal 0.05.  Under the
fixed cyclic alternative, \(\delta=14.051\), analytic power is 0.8401, and
empirical power is 0.8443.  The precision-orthogonal cyclic reconstruction has
relative \(L^2\) error \(1.31\times10^{-15}\).  We also evaluate Fisher
information using the Brunet--Derrida-corrected speed.

The second module was verified on an NVIDIA A100-SXM4-80GB using a CUDA 12.8
compiler.  A cell-centered spherical-polar \(128^3\) finite-volume solver runs
2,000 Strang-split reaction--diffusion--reaction steps.  Its selected Pareto
point reduces the undesired population state by 0.99695 and retains 2.228 times the
zero-production beneficial baseline.  A PyTorch-2.3 reverse-mode run evaluates
32 designs for 240 steps over 100 gradient iterations and independently reaches
adverse-state reduction 0.99347 and baseline retention 3.579.  The geometry-only scaled
Jacobian has rank two, whereas adding endpoint and monoculture-like calibration
observables raises the rank to four; the combined condition number is 34.4 and
the normalized smallest Fisher eigenvalue is \(8.44\times10^{-4}\).
A requirement-level audit in the ancillary results therefore marks Modules 1
and 2 complete, but not the expanded six-module program.

For the third module, we downloaded fallback external data directly on the A100
server.  The official AGAR representative demo contains 40 image/annotation
pairs, 1,886 colony boxes, all five requested species classes, and zero invalid
boxes.  A deterministic DINOv2 ViT-L/14 smoke probe on plate-disjoint colony
crops produced 1,024-dimensional embeddings for 399 validation colonies and
macro-F1 0.522.  This is useful computer-vision plumbing evidence, not the
requested full AGAR result: the 5,241+1,747-image archive requires publisher
registration, and the strict AgarNet segmentation and physical MG1655
phenotype gates remain unavailable.

The computational portion of the fourth module uses a learned Deep-BSDE
controller for a three-species stochastic Lotka--Volterra model with
temperature, medium concentration, and inducer concentration as controls.
Across 800 optimization iterations, the terminal BSDE loss decreases from
3.804 to 0.0936.  Evaluation on 32,768 common-noise trajectories gives
adverse-state reduction 0.9906 and beneficial retention 1.217 relative to
the zero-inducer baseline.  The feedback policy is represented at twelve
two-hour decision points and its observed control derivatives remain strictly
within the registered ramp bounds.  These trajectories are computational
samples, not biological replicates.  As a data-pipeline fallback, we also
construct 60 deterministic AGAR-demo Copy-Paste tiles with fixed seeds,
recorded source hashes, and box-derived ellipse masks; no Segment Anything
checkpoint is used, so this is not claimed as the strict SAM augmentation.
The preregistered \(n\geq30\) validation, BCa intervals, BH FDR decisions, and
requested empirical BH/Bonferroni power ratio remain unavailable until a
provenance-complete replicate table is received.

\paragraph{Genome-scale reconstruction and flux sampling.}
For the computational portion of the fifth module, CarveMe 1.6.6 reconstructs
an \emph{E.\ coli} K-12 MG1655 model directly from versioned RefSeq assembly
GCF\_000005845.2.  The resulting SHA-256-pinned SBML model contains 2,497
reactions, 1,561 metabolites, 1,606 genes, and 353 boundary reactions and has
an optimal feasible FBA solution.  A batched multi-chain artificial-centering
hit-and-run implementation then generates \(10^6\) flux-space transitions on
the A100-SXM4-80GB in 3.67 seconds after a 64.6-second CPU FVA/warmup stage.
Every transition is checked against the sparse equality system and variable
bounds; the final 2,048 chain states all pass COBRA feasibility.  A scalar
NumPy step and the vectorized CUDA step agree to machine precision, and a
same-seed full rerun gives a byte-identical exchange-flux summary.  These
correlated draws are not biological replicates.  Community FBA and the
remaining mechanism claims still require a provenance-complete multi-strain
manifest, measured phenotype tensor, aligned exchange/payoff series, targeted
deletion experiments, and physical calibrations.

As a fifth-module fallback, we also trained an edge-aware heterogeneous graph
transformer on the Nestor measured interaction table rather than on the absent
2,850 simulated co-culture set.  The table has 7,880 unique
strain-pair/carbon-condition rows, yielding 15,760 directional examples with
both directions of each pair kept in the same deterministic split.  The
768-dimensional, 12-head, 3-layer model has 16.9 million trainable parameters
and reaches test \(R^2=0.873\) on the held-out directional effects.  It is
therefore recorded as fallback evidence only.

\paragraph{Stochastic resolution scan.}
For a full radial boundary path, we scan contiguous intervals with an exact
Gaussian likelihood-ratio statistic and Bonferroni family-wise correction.  In
the current simulation, six independent boundary paths reach corrected
detection probability above \(0.9\) once episode drift exceeds approximately
0.235, and localization-aware detection follows the same ordering.

\paragraph{Public-image trace audit.}
As a provenance check, we also run the deterministic front end on the public
Figure 1 TIFF from Weinstein et al. \cite{weinstein2017genetic}, using only the
left real-colony panel.  Because this is a composite publication figure with
overlaid annotations, it is not used as biological ground truth.  The audit
tests whether fixed-rule support extraction, coarse color labeling, and radial
transition counting remain plausible on a real colony photograph.  The selected
panel has valid fixed-rule labels on \(83.6\%\) of support pixels and a median
of 12 boundary transitions across the trusted radial scan.

The sixth-module bridge is now partially implemented.  The public TIFF trace
audit passes, and the older multi-resolution endpoint gate accepts 13 of 15
endpoints.  An attention-based sequence regressor estimates \(q_{ij}\) without
Savitzky--Golay smoothing on synthetic boundary histories (RMSE \(0.0388\),
cyclic relative \(L_2=0.127\)), and radial-filtration TDA features classify the
15 endpoint traces with accuracy \(1.0\) while preserving the \(83.6\%\) public
TIFF label audit.  The tiny resolution-invariant ViT smoke probe remains a
failure: after resizing 1024--4096 px endpoints to 96 px and adding balanced
synthetic training images, held-out real endpoint accuracy is \(0.667\), and
all augmented views of the failed 1024 px trace are still called stable.  We
therefore do not claim the strict ViT bridge.

\begin{figure}[t]
\centering
\includegraphics[width=\linewidth]{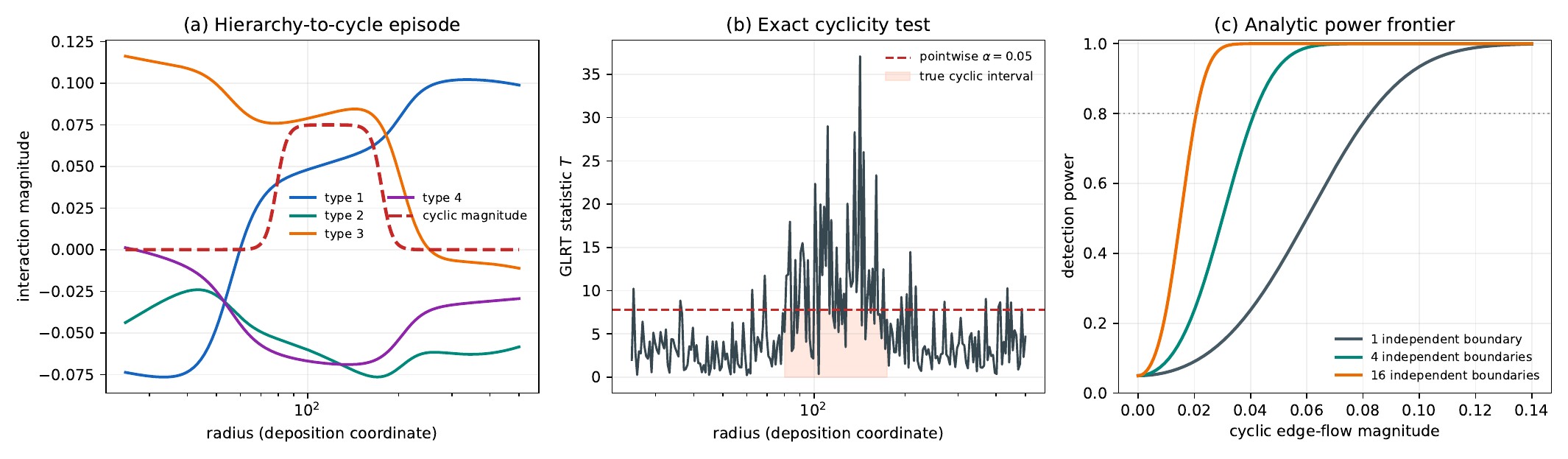}
\caption{Analytic benchmark: a hierarchy-to-cycle episode, exact cyclicity
test, and power frontier for repeated boundary observations.}
\label{fig:analytic}
\end{figure}

\begin{figure}[t]
\centering
\includegraphics[width=\linewidth]{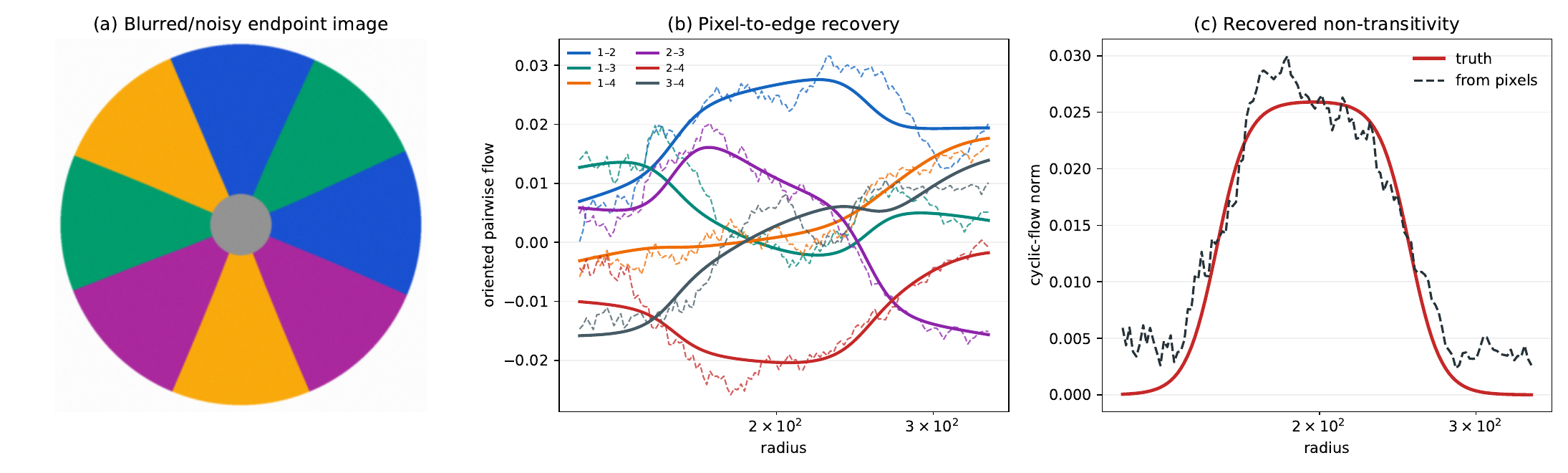}
\caption{Pixel-level endpoint round trip: blurred/noisy image, recovered
pairwise edge flow, and recovered cyclic-flow norm.}
\label{fig:pixel}
\end{figure}

\begin{figure}[t]
\centering
\includegraphics[width=\linewidth]{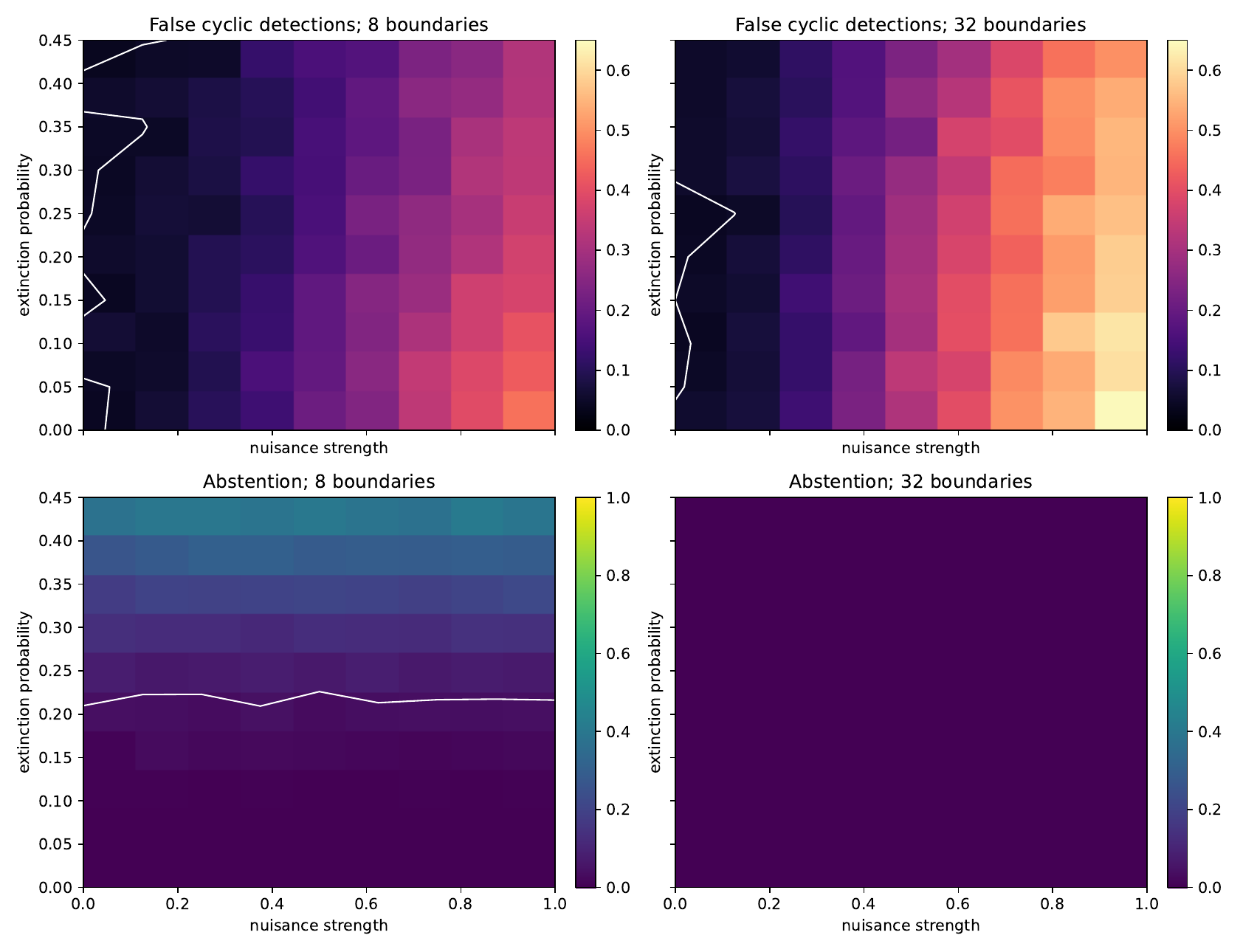}
\caption{Scalar-compatible null stress test.  The true game is always
transitive, \(q=Bf\).  High cyclic residuals in the upper panels therefore mark
frontier/imaging nuisance regimes where the method must abstain or demand
external calibration rather than infer non-transitive ecology.}
\label{fig:nullphase}
\end{figure}

\begin{figure}[t]
\centering
\includegraphics[width=\linewidth]{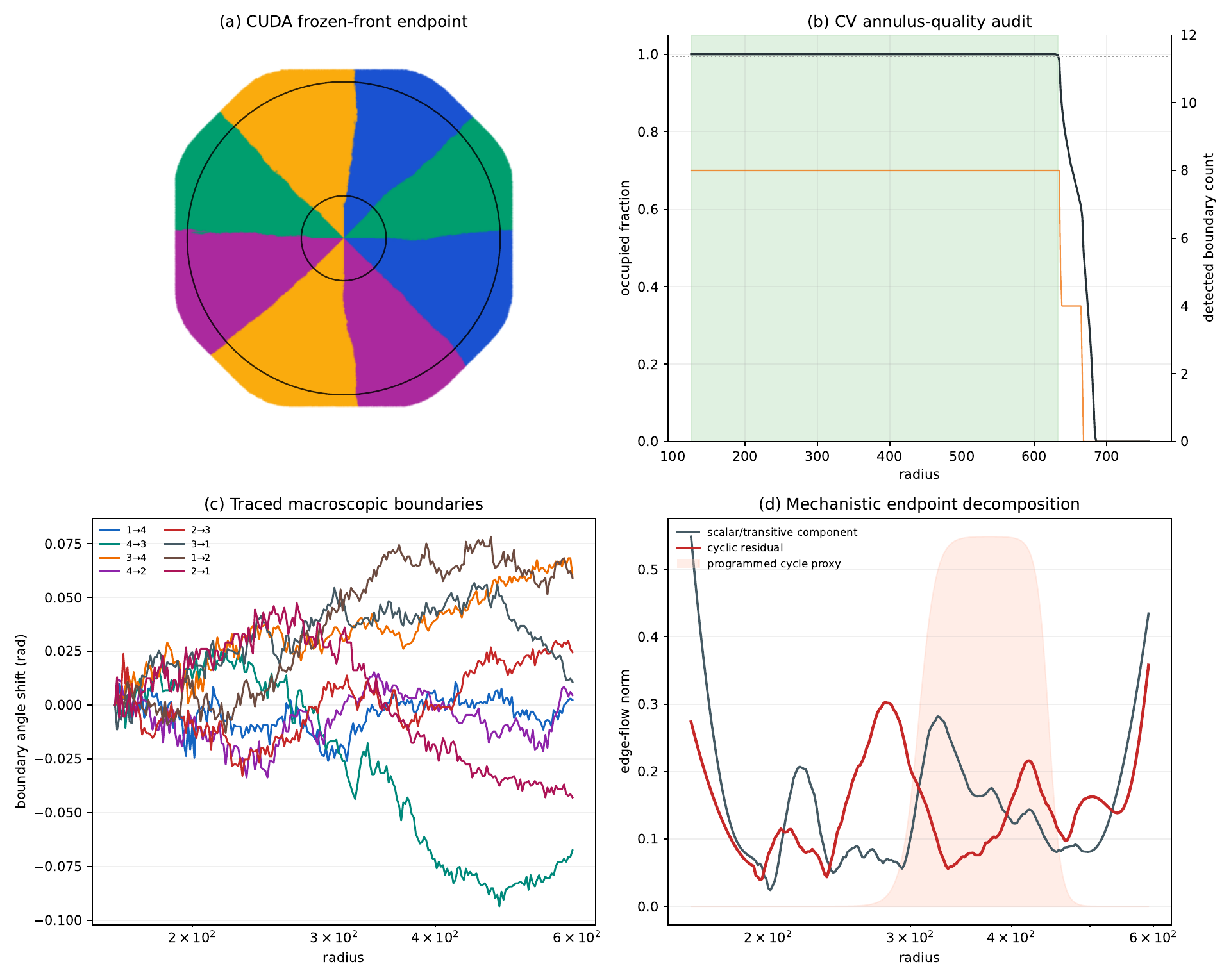}
\caption{Mechanistic endpoint analysis: raw frozen-front endpoint,
annulus-quality audit, traced boundaries, and transitive/cyclic decomposition.}
\label{fig:cuda}
\end{figure}

\begin{figure}[t]
\centering
\includegraphics[width=\linewidth]{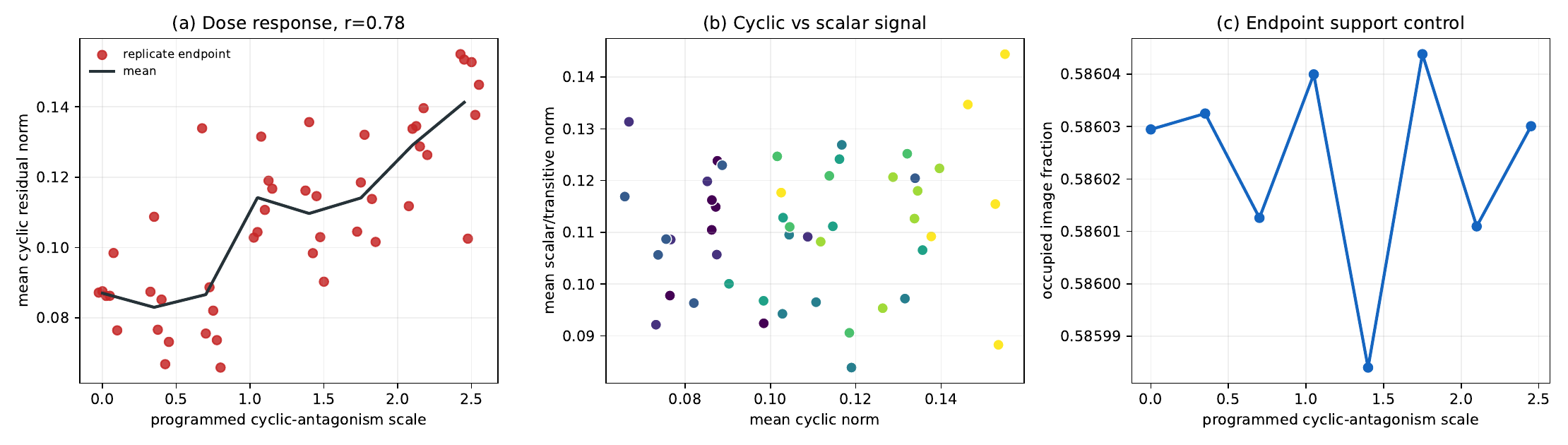}
\caption{Cyclic-antagonism dose response.  Accepted endpoint images show a
positive relationship between programmed non-transitive interaction strength and
the recovered cyclic residual, while endpoint support is stable.  Failed visual
stability gates are reported separately as abstentions.}
\label{fig:cyclesweep}
\end{figure}

\begin{figure}[t]
\centering
\includegraphics[width=\linewidth]{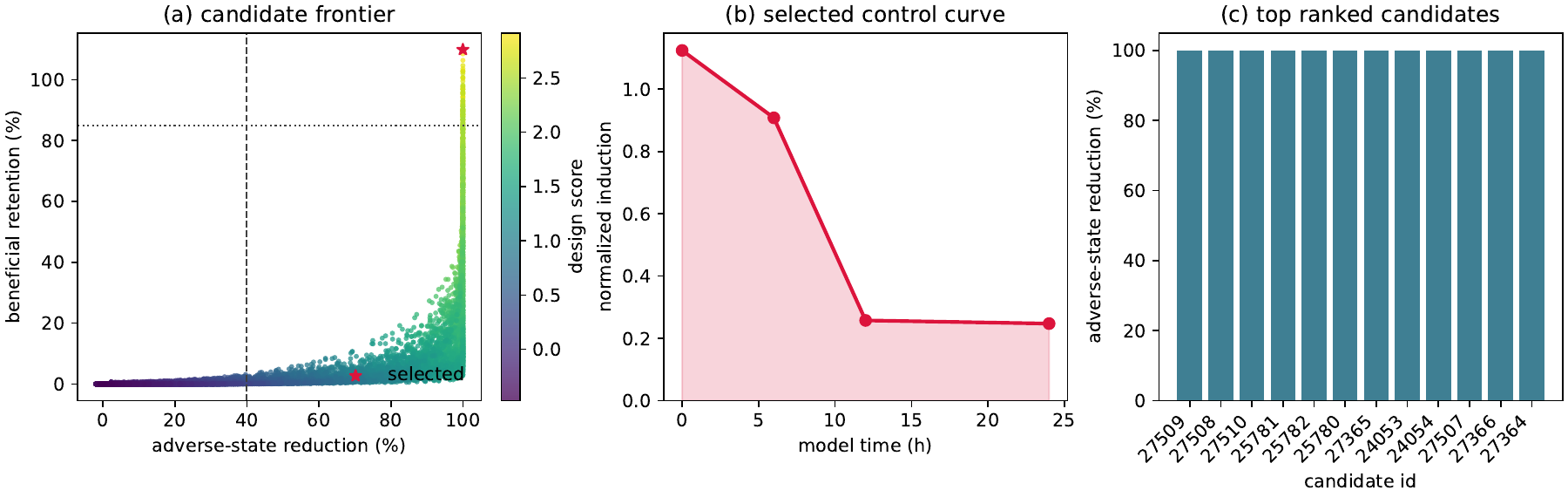}
\caption{Active design-control sweep.  The recovered cyclic residual is
used as a forcing target for a dimensionless reaction-diffusion control problem.
The selected candidate sits on the adverse-state-reduction and beneficial-retention
frontier and yields an early-pulse normalized induction curve.}
\label{fig:activecontrol}
\end{figure}

\begin{figure}[t]
\centering
\includegraphics[width=\linewidth]{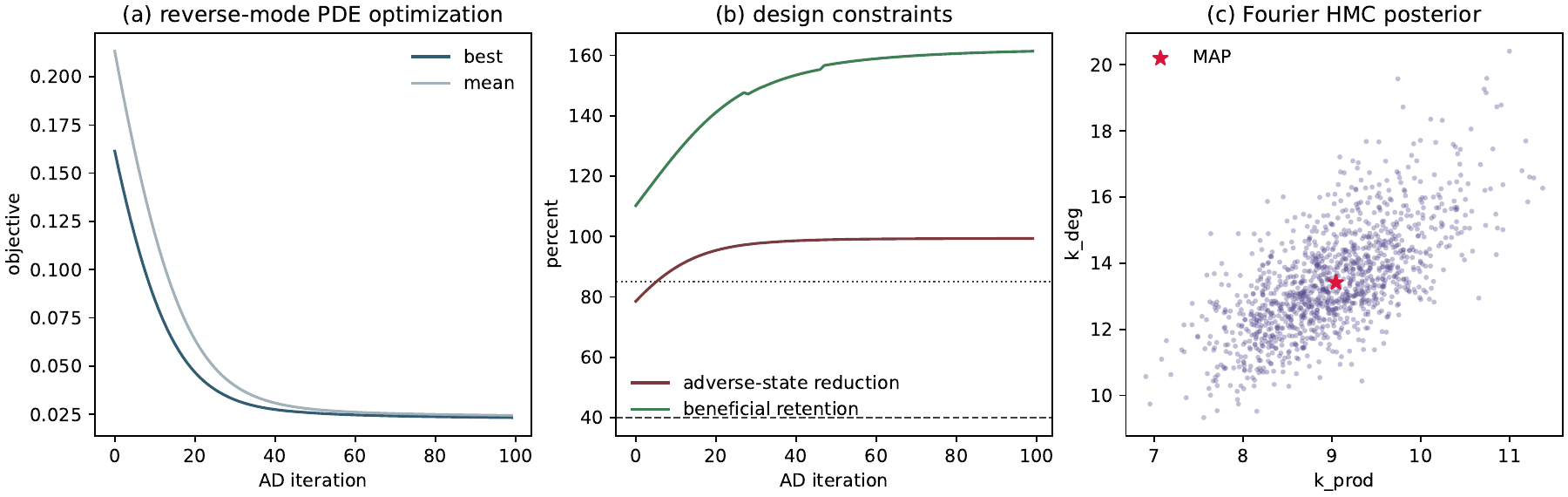}
\caption{Adjoint-state active design and Fourier HMC inversion.  Reverse-mode
automatic differentiation through the unrolled PDE reduces the objective while
maintaining the beneficial-strain constraint, and HMC estimates production and
degradation rates from the cyclic-residual spectrum.}
\label{fig:adjointhmc}
\end{figure}

\begin{figure}[t]
\centering
\includegraphics[width=\linewidth]{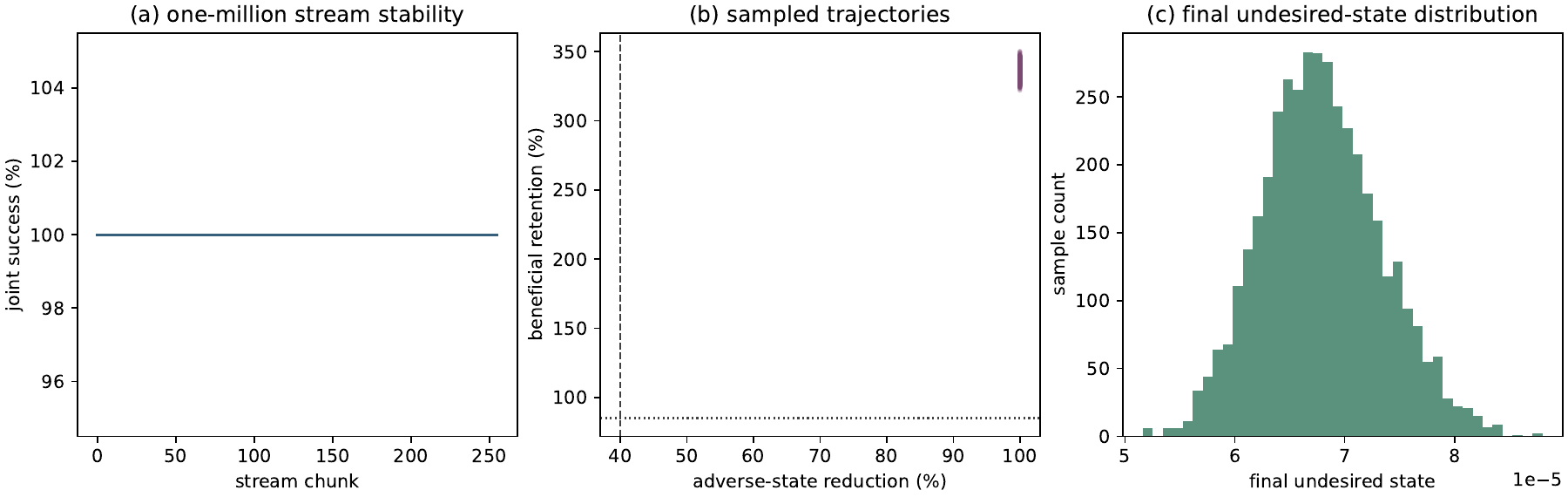}
\caption{One-million Monte Carlo active-design robustness run.  The stochastic
moment-projected trajectories test the adjoint-selected design under parameter
and process perturbations, while chunk summaries verify stable throughput and
success rates across the full stream.}
\label{fig:mccontrol}
\end{figure}

\begin{figure}[t]
\centering
\includegraphics[width=\linewidth]{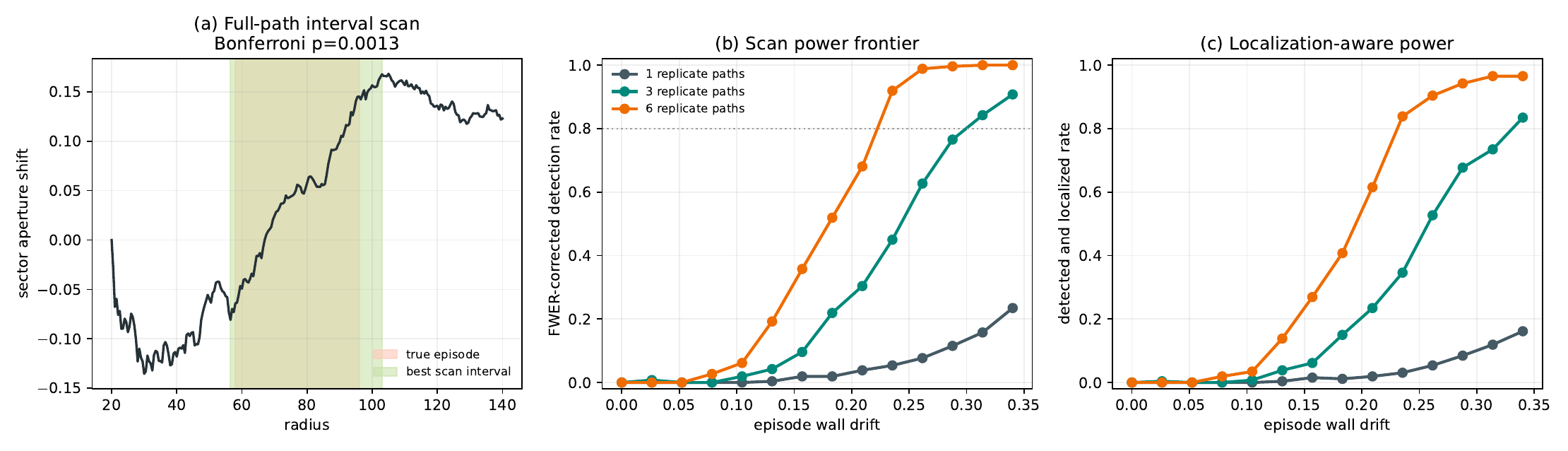}
\caption{Full-path stochastic resolution benchmark.  The interval scan uses the
entire radial boundary path and controls the scanned-interval family by a
Bonferroni correction.}
\label{fig:stochastic}
\end{figure}

\begin{figure}[t]
\centering
\includegraphics[width=\linewidth]{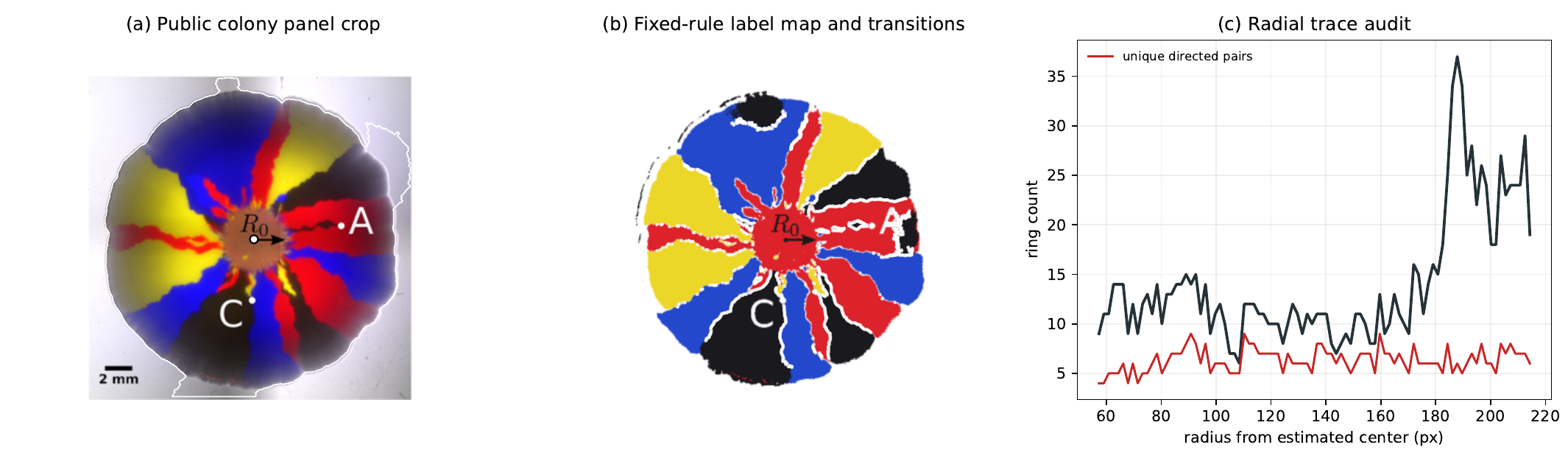}
\caption{Public-image trace audit on the left real-colony panel of a published
composite figure from Weinstein et al.  The output is a computer-vision sanity
check only: annotations and publication processing prevent treating it as raw
quantitative microscopy.}
\label{fig:publicaudit}
\end{figure}
\FloatBarrier

\section{Translation and dissemination boundary}

The biological motivation includes local antagonistic interactions in microbial
range expansions.  In this version, the computational
object is not limited to passive interpretation: cyclic residuals become inputs
to active design-control simulations over interaction matrices, strain
parameters, and admissible control schedules.  Laboratory translation still
requires independent calibration, approved host systems, institutional review,
and explicit separation between dimensionless model variables and bench
execution details.  This boundary is consistent with the broader dual-use
concerns around life-sciences dissemination
\cite{kuhlau2011precautionary,nasem2017dualuse,who2022responsible,whitehouse2025biosafety,usg2024durc}.

\section{Limitations}

The implementation assumes a known or well-estimated expansion center, stable
macroscopic sectors, visible pair contacts, and a deposition coordinate close
enough to radial order.  Topological events, overgrowth, strong chirality, front
anisotropy, post-deposition rearrangement, and sector extinction define the
abstention regime.  Public biological data enter the quantitative benchmark
only with clear provenance, license, calibration, and safety context.

\section{Conclusion}

Radial interaction tomography reframes a microbial colony endpoint image as a
classical pattern-recognition record of pairwise evolutionary interaction.  The
durable contribution is the combination of endpoint-image inversion, explicit
non-identifiability, contact-complete experiment design, stochastic cyclicity
testing, and active design-control simulation.  The benchmark now exercises the
full path from pixels to graph-Hodge residuals to dimensionless control
objectives; external colony-image validation extends that path when provenance,
calibration, and governance context are explicit.

\bibliographystyle{plain}
\bibliography{refs}

\clearpage
\appendix
\section{Proofs}

\begin{theorem}[Endpoint observation under a frozen radial sector model]
Fix an annulus \(\rho\in[a,b]\) and a known expansion center.  Suppose a labeled
endpoint image is generated by a circular sector order
\((s_1,\ldots,s_L)\) and \(C^1\) non-crossing boundary graphs
\(\theta_1(\rho)<\cdots<\theta_L(\rho)<\theta_1(\rho)+2\pi\), with no
post-deposition label transport.  Then the labeled endpoint image determines
the oriented boundary traces \(\theta_\ell(\rho)\) and hence
\(q_\ell(\rho)=d\theta_\ell/d\rho\) on every differentiability point.  If the
observed trace has uniform angular error at most \(\varepsilon\), and
\(\theta_\ell\) has second derivative bounded by \(M\), a local finite
difference estimator using samples separated by \(h\) satisfies
\[
|\widehat q_\ell(\rho)-q_\ell(\rho)|\leq 2\varepsilon/h + Mh/2
\]
away from the interval endpoints.
\end{theorem}

\begin{proof}
At a fixed \(\rho\), the circle intersects the sector boundaries exactly at the
ordered angles
\[
\theta_1(\rho),\ldots,\theta_L(\rho).
\]
Between consecutive boundary angles the label is constant and equal to the
corresponding sector type.  Therefore the discontinuities of the circular label
function are exactly
the boundary angles, with their oriented adjacent label pairs.  Because the
graphs are non-crossing and \(C^1\), ordering at one radius uniquely associates
each discontinuity with the same boundary at all nearby radii; continuation over
\([a,b]\) gives the full traces.  Differentiability of \(C^1\) traces gives
\(q_\ell=d\theta_\ell/d\rho\).

For the stability bound, let the measured trace be
\(\widehat\theta(\rho)=\theta(\rho)+e(\rho)\) with \(|e|\leq\varepsilon\).  The
one-sided estimator is
\[
\widehat q(\rho)=\frac{\widehat\theta(\rho+h)-\widehat\theta(\rho)}{h}.
\]
The error from measurement perturbation is at most
\((|e(\rho+h)|+|e(\rho)|)/h\leq2\varepsilon/h\).  Taylor's theorem gives
truncation error at most \(Mh/2\).  Adding the two terms proves the stated
bound.  Centered differences improve constants but are not needed for the
identifiability statement.
\end{proof}

\begin{theorem}[Weighted transitive/cyclic decomposition]
Let \(G=(V,E)\) be connected, \(B\in\mathbb{R}^{|E|\times |V|}\) its oriented
incidence matrix, and \(W\) a positive diagonal matrix.  Every edge-flow vector
\(q\) has a unique decomposition
\[
q=Bf+c,\qquad B^\top Wc=0,\qquad \mathbf 1^\top f=0.
\]
The cyclic subspace has dimension \(|E|-|V|+1\).
\end{theorem}

\begin{proof}
Consider the strictly convex least-squares problem on the gauge-fixed subspace
\(\mathbf 1^\top f=0\):
\[
\min_f (q-Bf)^\top W(q-Bf).
\]
Since \(G\) is connected, \(\mathrm{rank}(B)=|V|-1\), and \(B\) is injective on
the gauge-fixed subspace.  Hence the minimizer is unique.  The normal equations
are \(B^\top W(q-Bf)=0\).  Setting \(c=q-Bf\) gives the decomposition and the
orthogonality condition.  If two decompositions existed, their difference would
be both a gradient and \(W\)-orthogonal to every gradient, hence would have zero
weighted norm; uniqueness follows.  Finally, rank-nullity gives
\(\dim(\mathrm{im}\,B) = |V|-1\), so the \(W\)-orthogonal complement inside
\(\mathbb{R}^{|E|}\) has dimension \(|E|-|V|+1\).
\end{proof}

\begin{lemma}[Cycle certificates for scalar incompatibility]
Let \(Z\) be any signed cycle-basis matrix for \(G\): each row records the
oriented edge incidences of one fundamental cycle.  Then an edge flow \(q\) is
scalar-compatible if and only if \(Zq=0\).  Therefore the cyclic residual in
the weighted decomposition is zero if and only if every directed cycle sum
vanishes.
\end{lemma}

\begin{proof}
For any node potential \(f\), the sum of edge differences around a closed cycle
telescopes, so \(ZBf=0\).  Hence \(\mathrm{im}(B)\subseteq\ker(Z)\).  A cycle
basis has \(|E|-|V|+1\) independent rows, so
\[
\dim\ker(Z)=|E|-(|E|-|V|+1)=|V|-1.
\]
Because \(G\) is connected, \(\dim\mathrm{im}(B)=|V|-1\).  The inclusion is
therefore equality: \(\mathrm{im}(B)=\ker(Z)\).  Thus \(Zq=0\) exactly when
\(q=Bf\) for some potential \(f\).  In the weighted decomposition, \(c=0\) is
equivalent to \(q=Bf\), so the same condition is equivalent to zero cyclic
residual.
\end{proof}

\begin{proof}[Proof of radial clock non-identifiability]
The endpoint trace records the accumulated boundary angle
\(\theta(\rho)=\theta(a)+\int_a^\rho q(u)\,du\).  Let \(t=T(\rho)\) be any
strictly increasing differentiable physical clock and let
\(\tilde t=h(t)\) for another increasing differentiable bijection \(h\).  The
rate in the new clock is
\[
\tilde a(\tilde t)=a(h^{-1}(\tilde t))\,\frac{d h^{-1}}{d\tilde t}.
\]
Changing variables in the integral shows
\(\int a(t)\,dt=\int \tilde a(\tilde t)\,d\tilde t\).  Thus the same endpoint
trace is produced by infinitely many physical clocks.  No endpoint-only
algorithm can choose among them.
\end{proof}

\begin{proof}[Proof of the minimum complete-contact design theorem]
A circular inoculation order is a closed walk on the genotype contact graph:
each boundary between adjacent sectors traverses one edge.  To observe an
unrestricted antisymmetric pairwise game, the walk must cover every edge of
\(K_n\).  If \(n\) is odd, every vertex of \(K_n\) has degree \(n-1\), which is
even, so \(K_n\) is Eulerian.  An Euler circuit covers each edge once and has
\(\binom n2\) boundaries; no shorter walk can cover all \(\binom n2\) edges.

If \(n\) is even, every vertex of \(K_n\) has odd degree.  Any closed walk has
even degree at every vertex in the multigraph of traversed edges.  Therefore the
walk must duplicate edges so that all \(n\) odd vertices become even.  One
duplicated edge changes parity at two vertices, so at least \(n/2\) duplicated
edges are necessary.  A perfect matching supplies exactly \(n/2\) duplicated
edges, after which all degrees are even and an Euler circuit exists in the
resulting multigraph.  The length is therefore \(\binom n2+n/2\), and the lower
bound is attained.
\end{proof}

\begin{theorem}[Exact Gaussian cyclicity test]
Suppose \(\widehat q\sim\mathcal N(Bf+c,\Sigma)\), where \(\Sigma\) is positive
definite and \(G\) is connected.  Testing \(H_0:c=0\) with
\[
T=\min_x(\widehat q-Bx)^\top\Sigma^{-1}(\widehat q-Bx)
\]
gives \(T\sim\chi^2_{|E|-|V|+1}\) under \(H_0\).  Under a fixed cyclic
alternative \(c\), \(T\) is noncentral chi-squared with noncentrality
\(\delta=c^\top\Sigma^{-1}c\).
\end{theorem}

\begin{proof}
Whiten the observation: \(y=\Sigma^{-1/2}\widehat q\) and
\(\widetilde B=\Sigma^{-1/2}B\).  Under \(H_0\), \(y\) is a unit-covariance
Gaussian with mean in \(\mathrm{im}(\widetilde B)\).  The statistic \(T\) is the
squared norm of the orthogonal projection of \(y\) onto
\(\mathrm{im}(\widetilde B)^\perp\).  Since \(G\) is connected,
\(\mathrm{rank}(\widetilde B)=|V|-1\), so the orthogonal complement has
dimension \(|E|-|V|+1\).  A squared norm of that many independent standard
normal coordinates is chi-squared.  With cyclic mean \(c\), the projected mean
has squared norm \(c^\top\Sigma^{-1}c\), giving the stated noncentral
chi-squared law.
\end{proof}

\begin{corollary}[Replicate power scaling]
Suppose \(R\) independent replicate edge-flow estimates satisfy
\(\widehat q_r\sim\mathcal N(Bf+c,\Sigma)\) with common covariance.  The
pooled cyclicity test has the same null law
\(\chi^2_{|E|-|V|+1}\) and noncentrality
\[
\delta_R = R\,c^\top\Sigma^{-1}c .
\]
At level \(\alpha\), its power is
\[
1-F_{\chi^2_k(\delta_R)}\left(\chi^2_{k,1-\alpha}\right),
\qquad k=|E|-|V|+1,
\]
where \(F_{\chi^2_k(\delta_R)}\) is the noncentral chi-squared CDF.
\end{corollary}

\begin{proof}
The replicate mean has distribution
\(\overline q\sim\mathcal N(Bf+c,\Sigma/R)\).  Applying the exact Gaussian
cyclicity theorem with covariance \(\Sigma/R\) gives the same residual
dimension \(k\) and noncentrality
\(c^\top(R\Sigma^{-1})c=R\,c^\top\Sigma^{-1}c\).  The stated power is the upper
tail probability above the central chi-squared level-\(\alpha\) threshold.
\end{proof}

\begin{theorem}[Full-path constant drift information]
For aperture increments
\[
\Delta\phi_k\sim\mathcal N\left(2w\,\frac{\Delta r_k}{m_k},
4D\,\frac{\Delta r_k}{m_k^2}\right),
\]
where \(m_k=(r_k+r_{k+1})/2\), the maximum-likelihood estimator of constant
wall drift \(w\) has Fisher information
\[
I(w)=\sum_k \frac{(2\Delta r_k/m_k)^2}{4D\Delta r_k/m_k^2}
=\frac{r_N-r_0}{D}.
\]
Thus \(\mathrm{se}(\widehat w)=\sqrt{D/(r_N-r_0)}\).
\end{theorem}

\begin{proof}
This is a one-parameter Gaussian linear model with known heteroscedastic
variance.  The Fisher information is the sum of squared designs divided by
variances.  Substituting the design and variance cancels \(m_k\) and leaves
\(\Delta r_k/D\).  Summing over increments yields \((r_N-r_0)/D\).  The
least-squares/MLE variance is the reciprocal information.
\end{proof}

\begin{theorem}[Bonferroni-valid interval scan]
For a fixed finite family \(\mathcal I\) of radial intervals, let \(p_I\) be the
exact two-sided Gaussian likelihood-ratio \(p\)-value for interval \(I\) under
the zero-drift null.  The scan rule that rejects when
\(\min_{I\in\mathcal I} |\mathcal I|p_I\leq\alpha\) controls family-wise false
positive probability at level \(\alpha\).
\end{theorem}

\begin{proof}
Each \(p_I\) is a valid null \(p\)-value, so
\(\Pr(p_I\leq \alpha/|\mathcal I|)\leq \alpha/|\mathcal I|\).  By the union
bound,
\[
\Pr\left(\min_I p_I\leq \alpha/|\mathcal I|\right)
\leq \sum_{I\in\mathcal I}\Pr(p_I\leq \alpha/|\mathcal I|)
\leq \alpha.
\]
No independence between overlapping intervals is required.
\end{proof}

\begin{proposition}[Local/global non-identifiability]
If observed edge flow is modeled as \(q=Bf+\ell\), where \(f\) is a global
front-speed potential and \(\ell\) is local pairwise interaction, then endpoint
geometry cannot uniquely separate \(f\) from \(\ell\).  For any gauge-fixed
vector \(h\), \(q=B(f+h)+(\ell-Bh)\) gives the same endpoint trace.
\end{proposition}

\begin{proof}
The endpoint geometry depends only on the sum \(Bf+\ell\).  Replacing \(f\) by
\(f+h\) and \(\ell\) by \(\ell-Bh\) leaves this sum unchanged.  Since there are
infinitely many admissible \(h\), the separation is not identifiable without
additional measurements or constraints.
\end{proof}

\end{document}